\documentclass[final]{colt2026} 

\title[Equivalence Queries, Revisited]{Learning from Equivalence Queries, Revisited}
\usepackage{times}
\usepackage{cleveref}
\usepackage{mdframed}
\usepackage{xcolor}
\newcommand{\ignore}[1]{}
\usepackage{forest}
\newcommand{\ce}{x \sim \mathrm{CE}(c,h \mid \text{history})}
\newtheorem*{theoremrest}{Theorem}
\newcommand{\E}{\mathop\mathbb{E}}
\newtheorem{claim}{Claim}
\newcommand{\labelthis}[1]{%
  \tag{\theequation}\stepcounter{equation}\label{#1}%
}

\coltauthor{
 \Name{Mark Braverman} \Email{mbraverm@gmail.com}\\
 \addr Princeton and Google
 \AND
 \Name{Roi Livni} \Email{rlivni@tauex.tau.ac.il}\\
 \addr Tel-Aviv University
 \AND
 \Name{Yishay Mansour} \Email{mansour.yishay@gmail.com}\\
 \addr Tel-Aviv University and Google
 \AND
 \Name{Shay Moran} \Email{smoran@technion.ac.il}\\
 \addr Technion and Google
 \AND
 \Name{Kobbi Nissim} \Email{Kobbi.Nissim@georgetown.edu}\\
 \addr Georgetown and Google%
}
\newcommand{\ym}[1]{\textcolor{magenta}{#1}}
\definecolor{ElectricPurple}{RGB}{180,0,255}
\newcommand{\rl}[1]{\textcolor{ElectricPurple}{#1}}
\begin{document}

\maketitle

\begin{abstract}
Modern machine learning systems, such as generative models and recommendation systems, often evolve through a cycle of deploying a model, observing user interactions, and updating the model intermittently based on feedback. This mode of learning contrasts with common supervised learning frameworks, which focus on loss or regret minimization over a shared sequence of prediction tasks.
Motivated by this deployment-driven learning cycle, we revisit the classical model of learning from equivalence queries, introduced by~\cite{angluin_eq}, which provides a simple abstraction of such interactions: a learner repeatedly proposes hypotheses and, whenever the deployed hypothesis is inadequate, receives a counterexample tailored to that hypothesis. Under fully adversarial counterexample generation, however, this model exhibits overly pessimistic worst-case behavior. Moreover, most existing work on learning from equivalence queries considers the \emph{full-information} setting, where the learner observes not only a counterexample but also its correct label. This is an assumption that does not always align with natural interactive settings.

To address these considerations, we restrict the environment to generate counterexamples in a less adversarial manner by introducing a broad class of counterexample generators, which we call \emph{symmetric}. Informally, such symmetric counter example generators select counterexamples based only on the symmetric difference between the hypothesis and the target, and encompass natural feedback mechanisms such as random counterexamples~\citep*{angluin_dohrn_2017,Bhatia22,CFR24}, as well as generators that select counterexamples minimizing a prescribed complexity measure over the instance space. Within this framework, we study learning from equivalence queries under both full-information and bandit feedback. We establish tight bounds on the number of learning rounds in both settings and outline directions for future research. Our techniques rely on a game-theoretic perspective on symmetric adversaries and combine adaptive weighting algorithms with minimax arguments.
\end{abstract}

\section{Introduction}
Much of the supervised learning literature focuses on settings in which a learner makes predictions on a sequence of tasks that is independent of the learner’s current model. For example, tasks drawn from a fixed distribution or presented obliviously by the environment.
In contrast, many modern learning systems are organized around the deployment of an explicit model that is exposed to users and only updated intermittently based on feedback that arises in response to that model. A language model, for instance, is often revised based on users flagging particular generated responses. In such settings, feedback is inherently tied to the currently deployed hypothesis rather than to a shared sequence of prediction tasks. 

Inspired by such learning scenarios, we revisit the classical model of learning from equivalence queries, originally introduced by~\cite{angluin_eq}.
The appeal of this model is that it offers a particularly clean and basic mathematical abstraction of the deployment-driven learning scenarios we wish to study, in which a learner proposes a model, receives feedback only when the model is inadequate, and then updates and redeploys a revised hypothesis.
\begin{mdframed}[
  backgroundcolor=gray!10,
  linecolor=gray!50,
  linewidth=0.6pt,
  roundcorner=6pt,
  innertopmargin=1em,
  innerbottommargin=1em
]
\begin{center}
\textbf{Learning from Equivalence Queries} \citep{angluin_eq}
\end{center}

\medskip
\noindent
Let $\mathcal{H}$ be an hypothesis class over a domain $\mathcal{X}$ with label space $\mathcal{Y}$, and let $c \in \mathcal{H}$ be an unknown target concept.
\medskip
\noindent
The learning interaction proceeds in rounds $t = 1,2,\ldots$ as follows:
\begin{itemize}
    \item The learner proposes an hypothesis $h_t \in \mathcal{H}$.
    \item The environment either \emph{accepts} or \emph{rejects} $h_t$.
    \item If the environment accepts $h_t$, the interaction terminates.
    \item If the environment rejects $h_t$, it returns a counterexample $(x_t,c(x_t))$ s.t.\
    $h_t(x_t) \neq c(x_t)$.
\end{itemize}

\noindent
The environment is required to accept $h_t$ whenever $h_t = c$.

\medskip
\noindent\textit{Interpretation.}
Acceptance indicates that the proposed hypothesis is deemed sufficient by the environment,
which may occur before exact identification of the target.
However, exact correctness is always recognized and must lead to acceptance.

\end{mdframed}

In its classical form, learning from equivalence queries can exhibit pessimistic worst-case behavior even for extremely simple hypothesis classes. A canonical example is the class of singletons over a domain of size $n$, where each hypothesis labels exactly one point as positive and all others as negative. In this case, an adversary can force any learning rule to make $n-1$ equivalence queries by repeatedly returning a counterexample on which the learner’s current singleton hypothesis predicts label~$1$ while the true label is~$0$. This strategy can be maintained consistently until only a single point remains.

This simple impossibility stems from two modeling choices. First, the learner is required to be proper: it must propose hypotheses from the given class. Indeed, if improper hypotheses were allowed, the learner could submit the all-zero hypothesis and terminate immediately. Second, the environment is allowed to respond with a worst-case counterexample at each round. 
For instance, in the singleton class, if the environment were to return a \emph{uniformly random} counterexample from the current disagreement set, then whenever the learner proposes an incorrect singleton, there are exactly two disagreeing points: the point labeled $1$ by the hypothesis (where the target label is $0$) and the target’s unique positive point (where the hypothesis predicts $0$). Thus, with probability $1/2$ the returned counterexample is the target’s positive point, revealing the correct singleton immediately; consequently, the interaction terminates after a constant expected number of rounds.

{In this work, we mitigate the pessimism of worst-case equivalence-query lower bounds by restricting the power of the counterexample-generating environment. We consider a class of adversaries that is weaker than fully adaptive worst-case adversaries, yet still allows nontrivial, history-dependent behavior; the formal definition appears in the next section. This setting strictly generalizes random counterexample generation, which was previously shown to admit efficient learning guarantees in equivalence-query models \citep{angluin_dohrn_2017,Bhatia22,CFR24}, while permitting substantially richer forms of interaction.}

{At the same time, we retain the requirement that the learner be \emph{proper}, i.e., that it propose hypotheses from the original hypothesis class. This requirement is motivated by several complementary considerations. First, in many learning systems, model updates are naturally implemented as parameter updates within a fixed architecture, which directly corresponds to remaining within a given hypothesis class. Second, from a computational perspective, properness can be viewed as a restriction to a class of \emph{efficiently evaluable} models: enforcing that the learner’s proposals come from a fixed hypothesis class ensures that inference at deployment time remains efficient - a concern that is particularly critical for large models operating at massive scale, where even modest increases in per-query computation can have substantial practical and resource costs. From a modeling perspective, properness also supports interpretability: in many parametrized classes, such as linear models, the magnitudes of parameters often carry semantic meaning, highlighting influential features or decision components.}

{Beyond interpretability considerations, improper learning can in fact produce hypotheses that do not correspond to any meaningful object in the underlying modeling domain. As a concrete illustration, consider learning a ranking over $k$ elements, where hypotheses correspond to total orders and counterexamples reveal misordered pairs. Identifying a ranking can be viewed as learning a binary relation (i.e., a set of ordered pairs) over the ground set. If one allows improper learning, a natural strategy is to maintain the set of rankings consistent with the observed counterexamples and to predict each pairwise comparison according to the majority vote among these rankings - an approach analogous to the halving algorithm and guaranteed to converge in $O(\log m)$ rounds, where $m$ is the size of the initial hypothesis class. However, the resulting prediction need not correspond to any total order: in particular, it may induce a directed cycle, as in the Condorcet paradox~\citep{Condorcet1785,arrow1950difficulty}. Such an output is syntactically valid from the perspective of the learning rule, yet lacks a coherent semantic interpretation as a ranking, and would naturally be rejected by a user as a hallucinated solution. Restricting the learner to propose proper hypotheses rules out such pathologies by construction.}

\paragraph{Bandit Feedback.}
In addition to studying the standard full-information equivalence-query setting, we also consider a weaker form of feedback. Rather than always assuming that the learner observes a full counterexample together with its correct label, we analyze settings in which the feedback is only partial. In particular, we consider the bandit-feedback regime, where the learner is given an unlabeled counterexample but is not informed of its correct label. While for binary labels the models are equivalent, for a multi-class setting, when the set of labels is large, this introduces a new challenge. This type of feedback naturally models user interactions such as negative reactions or rejection signals, for example when a model hallucinates or generates inconsistent content, in which case users typically flag the response as wrong without providing an explicit correction or an alternative correct output.

\begin{mdframed}[
  backgroundcolor=gray!5,
  linecolor=gray!40,
  linewidth=0.5pt,
  roundcorner=6pt,
  innertopmargin=0.8em,
  innerbottommargin=0.8em
]
\begin{center}
\textbf{Learning from Equivalence Queries with Bandit Feedback}
\end{center}
\medskip
\noindent
The learning interaction is identical to the equivalence-query model above, except that when the proposed hypothesis $h_t$ is incorrect, the environment returns an \emph{unlabeled} counterexample~$x_t$ satisfying $h_t(x_t) \neq c(x_t)$, without revealing the correct label $c(x_t)$.
\end{mdframed}

\paragraph{Overview of our contributions.}
Motivated by the above considerations, we revisit learning from equivalence queries under two feedback models: the classical full-information setting, and the weaker bandit-feedback setting. To the best of our knowledge, this is the first study of learning from equivalence queries under bandit feedback. 

To circumvent classical impossibility results for even simple hypothesis classes, while still modeling natural example-generating mechanisms, we introduce the class of \emph{symmetric adversaries}. This broad family of counterexample generators treats the learner’s hypothesis and the target concept interchangeably, and captures several natural feedback mechanisms, including random counterexamples and counterexamples selected according to fixed orderings or simplicity criteria.

Within this framework, we show that the \emph{Littlestone dimension} precisely characterizes the query complexity of learning from equivalence queries under full-information feedback: the optimal expected query complexity against symmetric adversaries is $\Theta(\mathrm{Ldim}(\mathcal{H}))$, and this bound is tight for every hypothesis class.
In the bandit-feedback setting, we establish a $\tilde\Theta(k \cdot \mathrm{Ldim}(\mathcal{H}))$ bound, where $k$ denotes the effective label complexity of the class; this bound is optimal in the sense that it is achieved by matching lower bounds for some hypothesis classes. 

Finally, our perspective also highlights several directions for future research, concerning alternative feedback models and finer structural parameters governing equivalence-query learning.

\paragraph{Organization of the paper.}
The remainder of the paper is organized as follows.
In Section~\ref{sec:main-results}, we formally define symmetric counterexample generators and state our main results for both the full-information and bandit feedback settings, together with several remarks and directions for future research.
In Section~\ref{sec:proof-overview}, we provide a technical overview of the proofs.
{The complete proofs are deferred to the appendix. Additional related work beyond that discussed in the main text is also deferred to the appendix, and is presented in Appendix~\ref{sec:addRelated}.}

\section{Main Results}\label{sec:main-results}
We begin by formally defining the class of symmetric counterexample generators, and then proceed to state our main results for both the full-information and bandit feedback settings. 

\subsection{Symmetric Counterexample Generators}

Formally, a counterexample generator is a (possibly randomized) mapping that, given the interaction history so far - namely, the previously proposed hypotheses $h_1,\ldots,h_{t-1}$ - together with a target concept $c \in \mathcal{H}$ and the learner’s current hypothesis $h \in \mathcal{H}$ with $h \neq c$, outputs a counterexample
\[
x \;=\; \mathrm{CE}(h,c \mid h_1,\ldots,h_{t-1}) \in \mathcal{X}
\quad\text{such that}\quad
h(x) \neq c(x).
\]
When the generator is randomized, $\mathrm{CE}(h,c \mid h_1,\ldots,h_{t-1})$ denotes a distribution over such counterexamples.

\begin{definition}[Symmetric counterexample generators]
A counterexample generator is said to be \emph{symmetric} if for every pair of distinct hypotheses
$h,c \in \mathcal{H}$, and for every sequence $h_1,\ldots,h_{t-1}\in \mathcal{H}$,
\[
\mathrm{CE}(h,c \mid h_1,\ldots,h_{t-1})
\;=\;
\mathrm{CE}(c,h \mid h_1,\ldots,h_{t-1}),
\]
where equality is understood in distribution when the generator is randomized.
Equivalently, a symmetric counterexample generator treats the learner’s hypothesis and the target concept interchangeably.
\end{definition}

\begin{definition}[Order-induced counterexample generators]
A natural subclass of symmetric counterexample generators is given by \emph{order-induced counterexample generators}.
Such a generator is specified by a sequence
\(x_1, x_2, \ldots\) of elements in $\mathcal{X}$.
Given a learner hypothesis $h$ and a target concept $c$, if $h \neq c$ the generator returns the first element in the sequence on which $h$ and $c$ disagree, namely
\[
x = x_i \quad\text{where}\quad i = \min\{j : h(x_j) \neq c(x_j)\}.
\]
If no such index exists (i.e., the sequence contains no disagreement), the generator accepts and the interaction terminates.
More generally, we also allow randomized order-induced counterexample generators, specified by a probability distribution $\mathcal{D}$ over such sequences:
the generator samples a sequence from $\mathcal{D}$ and then applies the rule above.
\end{definition}
For brevity, we will sometimes refer to a \emph{symmetric adversary} or an \emph{order-induced adversary} to mean that the \emph{counterexample-generation mechanism} satisfies the corresponding symmetry or order-induced property.\footnote{
Formally, an adversary may have additional power beyond counterexample generation, including the ability to choose the target concept, and in adaptive settings even to modify it across rounds, provided that the chosen target remains consistent with all counterexamples revealed so far.
Our results apply to such adaptive adversaries; the only restriction we impose is that the counterexample-generation mechanism itself is symmetric, see {Appendix~\ref{sec:adaptiveadv} for more details.}}

Order-induced counterexample generators are symmetric by construction, and capture a broad range of natural counterexample-selection rules, including random counterexamples (see below), as well as generators that return the ``simplest'' counterexample according to a given complexity measure over instances. We now illustrate these definitions with two examples: the first is the random counterexample generator, previously studied by~\cite{angluin_dohrn_2017,Bhatia22,CFR24}, which is a special case of an order-induced counterexample generator; the second is a class of decision-tree--induced counterexample generators, which are symmetric by construction and will be useful later in our analysis.

\paragraph{Example 1 (Random counterexample generators~\citep{angluin_dohrn_2017}).}
Let $\mu$ be a probability distribution over the domain $\mathcal{X}$.
A \emph{random counterexample generator} responds to an incorrect hypothesis $h$ against a target concept $c$ by sampling a point
$x \sim \mu$ conditioned on the event $h(x) \neq c(x)$, assuming that the symmetric difference of $h$ and $c$ has positive $\mu$-measure.
Such a generator is order-induced: sampling $x$ from $\mu$ conditioned on $h(x) \neq c(x)$ is equivalent to drawing an infinite i.i.d.\ sequence
$x_1, x_2, \ldots \sim \mu$ and returning the first element in the sequence that lies in the symmetric difference of $h$ and $c$.

\paragraph{Example 2 (Decision-tree--induced counterexample generators).}
A \emph{decision tree} is a rooted tree whose internal nodes are labeled by instances from $\mathcal{X}$ and whose outgoing edges are labeled by elements of the label space $\mathcal{Y}$.
Each hypothesis $h \in \mathcal{H}$ induces a (possibly partial) path in the tree: starting at the root, at each visited internal node labeled by $x \in \mathcal{X}$, the path follows the outgoing edge labeled by $h(x)$, and terminates if no such edge exists or a leaf is reached (see Figure~\ref{fig:decision-tree-ternary}).
The \emph{decision-tree--induced counterexample generator} associated with such a decision tree is defined as follows.
Given a target concept $c$ and a hypothesis $h$, consider the paths induced by $c$ and $h$.
If the two paths terminate at the same node, the generator accepts and the interaction terminates.
Otherwise, the generator returns the instance $x \in \mathcal{X}$ labeling the first node at which the two paths diverge.
This generator is symmetric, since the returned instance depends only on the unordered pair $\{h,c\}$ and the fixed decision tree.

\begin{figure}[t]
\centering
\begin{forest}
for tree={
  draw,
  circle,
  inner sep=2pt,
  s sep=7mm,
  l sep=10mm,
  edge={-latex},
}
[$x_1$
  [$x_2$, edge label={node[midway,left,draw=none,fill=white,inner sep=1pt] {$0$}}
    [, edge label={node[midway,left,draw=none,fill=white,inner sep=1pt] {$0$}}]
    [, edge label={node[midway,draw=none,fill=white,inner sep=1pt] {$1$}}]
    [, edge label={node[midway,right,draw=none,fill=white,inner sep=1pt] {$2$}}]
  ]
  [$x_3$, edge label={node[midway,draw=none,fill=white,inner sep=1pt] {$1$}}
    [, edge label={node[midway,left,draw=none,fill=white,inner sep=1pt] {$0$}}]
    [, edge label={node[midway,draw=none,fill=white,inner sep=1pt] {$1$}}]
    [, edge label={node[midway,right,draw=none,fill=white,inner sep=1pt] {$2$}}]
  ]
  [$x_4$, edge label={node[midway,right,draw=none,fill=white,inner sep=1pt] {$2$}}
    [, edge label={node[midway,left,draw=none,fill=white,inner sep=1pt] {$0$}}]
    [, edge label={node[midway,draw=none,fill=white,inner sep=1pt] {$1$}}]
    [, edge label={node[midway,right,draw=none,fill=white,inner sep=1pt] {$2$}}]
  ]
]
\end{forest}
\caption{A schematic example of a (ternary) decision tree when $\mathcal{Y}=\{0,1,2\}$. Internal nodes are labeled by instances, and edges are labeled by outcomes in $\mathcal{Y}$. Leaves correspond to prediction paths and carry no instance labels.}
\label{fig:decision-tree-ternary}
\end{figure}
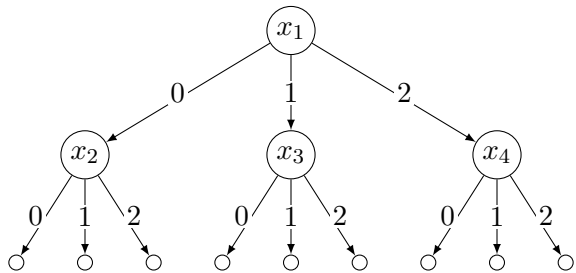
\subsection{Statement of Results}

We begin by recalling the Littlestone dimension, which will serve as the central complexity parameter for characterizing and quantifying the optimal number of rounds in learning from equivalence queries against symmetric adversaries. {We use the general form of the definition that is applicable to both multiclass and binary classification.}

\begin{definition}[Littlestone Dimension]
A \emph{mistake tree} is a rooted \underline{\emph{binary}} decision tree whose internal nodes are labeled by instances from $\mathcal{X}$.
Each internal node labeled by $x \in \mathcal{X}$ has exactly two outgoing edges, corresponding to two possible prediction outcomes.
Each hypothesis $h \in \mathcal{H}$ induces a unique root-to-leaf path in the tree by following, at each internal node labeled by $x$, the outgoing edge corresponding to the prediction~$h(x)$.
The \emph{Littlestone dimension} $\mathrm{Ldim}(\mathcal{H})$ is defined as the maximum depth of a binary mistake tree that is shattered by $\mathcal{H}$, meaning that for every root-to-leaf path in the tree there exists a hypothesis in $\mathcal{H}$ whose induced path coincides with it.
\end{definition}

The Littlestone dimension is a fundamental parameter in learning theory: it characterizes optimal mistake and regret bounds in multiclass online learning~\citep{Lit88,Ben-DavidPS09,AlonBDMNY21,Daniely15a,Hanneke23b}, the optimal number of rounds in improper learning from equivalence queries under fully adversarial counterexample generation~\citep{Lit88}, and it also plays a central role in characterizing differentially private PAC learnability~\citep{AlonLMM19,BunLM20,AlonBLMM22,GhaziG0M21,FioravantiHMST24,Lyu25}.



\begin{mdframed}[
  backgroundcolor=gray!10,
  linecolor=gray!50,
  linewidth=0.6pt,
  roundcorner=6pt,
  innertopmargin=1em,
  innerbottommargin=1em
]
\begin{theorem}[Full-information]\label{thm:full}
Let $\mathcal{H}$ be a finite hypothesis class with $\mathrm{Ldim}(\mathcal{H}) = d$, and let $\mathcal{A}$ be a symmetric adversary.
Then, there exists a learning rule such that for every target concept $c \in \mathcal{H}$, the interaction between the learner and $\mathcal{A}$ terminates after at most $O(d)$ equivalence queries in expectation.

Conversely, there exists a symmetric adversary $\mathcal{A}^\star$ such that, against $\mathcal{A}^\star$, any learning rule requires at least $\Omega(d)$ equivalence queries in expectation.
\end{theorem}
\end{mdframed}

The special case in which the adversary returns random counterexamples was studied by~\cite{angluin_dohrn_2017}, who proved an $O(\log|\mathcal{H}|)$ upper bound, and by~\cite{Bhatia22}, who provided an alternative algorithm achieving similar guarantees. More recently,~\cite{CFR24} extended the approach of~\cite{angluin_dohrn_2017} to obtain an $\mathrm{Ldim}$-based upper bound for random counterexamples. Our proof applies uniformly to all symmetric adversaries and relies on a game-theoretic formulation together with the minimax theorem, yielding a randomized learning rule. In the special case of random counterexamples, where prior algorithms are deterministic, we additionally provide an alternative proof for a deterministic learning rule. Our proof is shorter and simple, and is presented in the proof overview section. We next turn to the bandit-feedback setting.

\begin{mdframed}[
  backgroundcolor=gray!10,
  linecolor=gray!50,
  linewidth=0.6pt,
  roundcorner=6pt,
  innertopmargin=1em,
  innerbottommargin=1em
]
\begin{theorem}[Bandit feedback]\label{thm:bandit}
Let $\mathcal{H}$ be a finite hypothesis class with $\mathrm{Ldim}(\mathcal{H}) = d$, and let~$\mathcal{A}$ be a symmetric adversary.
Assume the label space $Y$ has size $\lvert Y\rvert = k$.
Then, there exists a learning rule such that for every target concept $c \in \mathcal{H}$, the interaction between the learner and $\mathcal{A}$ terminates after at most $O(d \cdot k\cdot\log k)$ equivalence queries with bandit feedback in expectation.

This bound is tight in the following sense: for every $k$ and $d$, there exists an hypothesis class $\mathcal{H}$ over a label space of size $k$ and a symmetric adversary $\mathcal{A}^\star$ such that, for any learning rule, the interaction between the learner and $\mathcal{A}^\star$ requires at least $\Omega(d \cdot k\cdot\log k)$ equivalence queries with bandit feedback in expectation.
\end{theorem}
\end{mdframed}
To the best of our knowledge, the bandit-feedback variant of learning from equivalence queries has not been studied previously.
The algorithm and analysis in this setting are considerably more involved than in the full-information case.
They rely on a careful adaptive weighting scheme over an appropriate collection of splitting experts, and require finer control over the limited information revealed by each counterexample.
We provide further details in the proof overview.


\paragraph{Infinite classes.}
Although our bounds do not depend on the domain size or the hypothesis class, our theorems assume that the hypothesis class $\mathcal{H}$ is finite.
This assumption is essential, as illustrated by the following simple example.
Let $\mathcal{H}$ be the class of singleton concepts over the natural numbers
$\mathbb{N}=\{1,2,3,\ldots\}$, where each hypothesis labels exactly one point as positive.
Note that the Littlestone dimension of $\mathcal{H}$ is 1.
Let the adversary be a random counterexample generator with respect to a distribution
$\mu$ over $\mathbb{N}$ that has full support, for example $\mu(n)=2^{-n}$. 
Consider any learning rule~$\mathcal{L}$; for simplicity, assume that $\mathcal{L}$ is deterministic.
Let $n_1$ be such that the learner’s first hypothesis is the singleton
$h_1=\mathbf{1}_{\{x=n_1\}}$.
For any $\varepsilon>0$, there exists $n_\varepsilon$ sufficiently large such that if the target
concept is $c=\mathbf{1}_{\{x=n_\varepsilon\}}$, then with probability at least $1-\varepsilon$
the returned counterexample is $x=n_1$, yielding labeled feedback $(n_1,0)$.
Let $n_2,n_3,\ldots$ denote the indices of the subsequent singleton hypotheses proposed by the
learner under the repeated feedback $(n_1,0),(n_2,0),\ldots$.
By choosing the target index $n_\varepsilon$ sufficiently large relative to
$n_1,n_2,\ldots,n_k$, the adversary can ensure that, with arbitrarily high probability,
the interaction continues for at least $k$ rounds.
Thus, no uniform bound on the number of equivalence queries in expectation is possible for infinite hypothesis classes.
A similar argument holds when $\mathcal{L}$ is randomized, by considering the distributions over the indices $n_1,\ldots,n_k$ and picking the index of the target concept sufficiently far away in the tail of these distributions.

\paragraph{Future Research.}
A first natural open problem concerns the gap between the full-information and bandit-feedback settings.
In the full-information case, our characterization is tight in a strong sense: for every hypothesis class~$\mathcal{H}$, the optimal number of equivalence queries against symmetric adversaries is $\Theta(\mathrm{Ldim}(\mathcal{H}))$.
In contrast, in the bandit-feedback setting, while our upper bound of $O(\mathrm{Ldim}(\mathcal{H})\cdot k \cdot \log k)$ applies to all classes, the matching lower bound is currently witnessed only by specific constructions.
Bridging this gap and obtaining a tight, class-dependent characterization of the bandit-feedback query complexity remains an intriguing open question.
A natural candidate for such a characterization is the \emph{bandit Littlestone dimension}~\citep{DanielyH13}, which is known to characterize the equivalence-query complexity in the fully adversarial bandit setting.
We note that $O(\mathrm{Ldim}(\mathcal{H})\cdot k \cdot \log k)$~\citep{AuerLong1999},
{and thus an upper bound expressed in terms of the bandit Littlestone dimension would strictly refine our bound.}

{Due to space constraints, we defer further directions for future research to Appendix~\ref{sec:addfuture}.}

\section{Proof Overview}\label{sec:proof-overview}

\subsection{Upper Bounds}

We begin by outlining the proof of the upper bound in the full-information setting.
This case admits a simple and intuitive argument, and several of its ideas will reappear in the bandit-feedback setting discussed later.

\subsubsection{Full Information}

We begin with an intuitive special case that highlights the main idea of the proof.
We assume that the adversary returns a \emph{random counterexample}, drawn from an arbitrary distribution supported on the symmetric difference between the learner’s hypothesis and the target concept, and show that the expected number of equivalence queries is $O(\log |\mathcal{H}|)$.
{This special case was previously studied in the binary setting by~\citep{angluin_dohrn_2017,Bhatia22,CFR24}.}
We present an alternative proof that is short and conceptually simple, relying on a direct minimax argument rather than the more technical case analysis used in prior work.
Importantly, the Littlestone-dimension-based upper bound and the extension to arbitrary symmetric adversaries follow from the same argument with only minor modifications.

Let $V \subseteq \mathcal{H}$ denote the current \emph{version space}, i.e., the set of hypotheses consistent with all counterexamples observed so far.
Initially, $V=\mathcal{H}$.
It suffices to show that in each round the learner can select an hypothesis such that, with constant probability, the version space shrinks by a constant factor.
Indeed, this implies that, after $O(\log \lvert \mathcal{H}\rvert)$ rounds in expectation, only the target hypothesis remains in the version space.

\paragraph{A zero-sum game.}
Fix the current version space~$V$.
We define a zero-sum game between two players: the learner~$L$ and the adversary~$A$.
Both players have the same set of pure strategies, namely the hypotheses in~$V$.
For a labeled example $(x,y) \in \mathcal{X}\times\mathcal{Y}$, let
\(
V_{x \to y} \;:=\; \{h' \in V : h'(x)=y\}
\)
denote the subset of hypotheses in the current version space that predict label~$y$ on~$x$.
If the learner plays $h \in V$ and the adversary plays $c \in V$, the payoff to the adversary is defined as
\begin{equation}\label{eq:payoff}
\mathrm{Payoff}(h,c)
\;=\;
\Pr_{x \sim h \,\Delta\, c}
\Bigl[
\lvert V_{x \to c(x)} \rvert \;>\; \tfrac{1}{2}\lvert V\rvert
\Bigr],
\end{equation}
where $h \,\Delta\, c$ denotes the \emph{symmetric difference} between $h$ and $c$, that is, the set of instances
$x$ for which $h(x)\neq c(x)$.
The probability is taken with respect to the adversary’s distribution over counterexamples supported on this set.
In the degenerate case $h=c$, we define the payoff to be~$0$.

The learner seeks to minimize this payoff.
Indeed, the event inside the probability corresponds to the situation in which, upon receiving a counterexample~$x$,
a strict majority of the hypotheses in the current version space agree with the target concept~$c$ on~$x$.
Equivalently, the learner aims to maximize the probability of the complementary event, in which
$\lvert V_{x \to c(x)} \rvert \le \tfrac{1}{2}\lvert V\rvert$.
In this complementary event, the version space shrinks by a factor of at least~$2$, since all hypotheses that disagree with the target on~$x$ are eliminated.
\paragraph{A symmetry property.}
The key structural property of the game is the following inequality: for every $h,c \in V$,
\begin{equation}\label{eq:symmetry}
\mathrm{Payoff}(h,c) + \mathrm{Payoff}(c,h) \;\le\; 1.
\end{equation}
When $h=c$, this is immediate since the payoff is defined to be~$0$.
When $h \neq c$, observe that for a counterexample $x \sim h \,\Delta\, c$, the two events
``$\lvert V_{x \to c(x)} \rvert > \tfrac{1}{2}\lvert V\rvert$''
and
``$\lvert V_{x \to h(x)} \rvert > \tfrac{1}{2}\lvert V\rvert$''
are disjoint, since $h(x)\neq c(x)$.
Therefore, the probabilities of these two events sum to at most~$1$, proving~\eqref{eq:symmetry}.

By linearity, the same inequality holds for mixed strategies:
for any distributions $p,q$ over~$V$,
\[
\mathrm{Payoff}(p,q) + \mathrm{Payoff}(q,p) \;\le\; 1.
\]

\paragraph{Minimax argument.}
We now invoke von Neumann’s minimax theorem~\citep{vonNeumann1928,cesaBianchiLugosi2006}.
By the minimax identity, to show that the learner can guarantee adversary payoff at most~$\tfrac{1}{2}$, it suffices to show that for every (possibly mixed) adversary strategy~$q$, there exists a learner strategy~$p$ such that
\[
\mathrm{Payoff}(p,q) \;\le\; \tfrac{1}{2}.
\]
This follows directly from~\eqref{eq:symmetry} by taking $p=q$.
Indeed,
\[
2\,\mathrm{Payoff}(q,q)
\;\le\;
\mathrm{Payoff}(q,q)+\mathrm{Payoff}(q,q)
\;\le\;
1,
\]
and hence $\mathrm{Payoff}(q,q)\le\tfrac{1}{2}$.

We conclude that there exists a (possibly randomized) learning rule such that in each round, with probability at least~$\tfrac{1}{2}$, the version space shrinks by a factor of at least~$2$.
It follows that the interaction terminates after $O(\log \lvert \mathcal{H}\rvert)$ equivalence queries in expectation.
{In fact, this argument yields a slightly stronger guarantee: the same bound continues to hold even if the adversary adaptively changes the target concept across rounds, as long as the counterexample-generation mechanism remains random and each chosen target is consistent with all counterexamples revealed so far.}



\paragraph{General case.}
We next explain how the above argument extends from random counterexamples to arbitrary symmetric adversaries, and how the logarithmic dependence on $m$ is replaced by a bound in terms of the Littlestone dimension.

Let $V \subseteq \mathcal{H}$ denote the current version space, and write $d=\mathrm{Ldim}(V)$. We redefine the payoff of the game in~\eqref{eq:payoff} so that the learner is penalized when the counterexample leaves the adversary with a \emph{full-dimensional} continuation.
Formally, for $h,c \in V$, define
\begin{equation}\label{eq:payoff-ldim}
\mathrm{Payoff}(h,c)
\;:=\;
\Pr_{x \sim \mathcal{A}(\text{history},c,h)}
\Bigl[
\mathrm{Ldim}\!\bigl(V_{x \to c(x)}\bigr)=\mathrm{Ldim}(V)
\Bigr],
\end{equation}
with $\mathrm{Payoff}(h,h)=0$.
Here the probability is taken over the counterexample distribution induced by the adversary~$\mathcal{A}$. The motivation behind~\eqref{eq:payoff-ldim} is that $\mathrm{Ldim}(\cdot)$ serves as a refined measure of the ``size'' of the version space.
Indeed, the Littlestone dimension can be viewed as a surrogate for the logarithm of the cardinality of $V$: for finite classes it always satisfies $\mathrm{Ldim}(V)\le \log |V|$, yet it may be much smaller and remains meaningful even when $V$ is infinite.
Accordingly, driving $\mathrm{Ldim}(V)$ down by one represents concrete progress in the learning process, regardless of the actual size of the version space.

The crucial observation is simple.
Fix $h\neq c$ and any counterexample $x$ with $h(x)\neq c(x)$.
Then it cannot be that \emph{both} subsets $V_{x \to h(x)}$ and $V_{x \to c(x)}$ have Littlestone dimension~$d$.
Indeed, if $\mathrm{Ldim}(V_{x \to h(x)})=\mathrm{Ldim}(V_{x \to c(x)})=d$, then $V$ would shatter a mistake tree of depth $d+1$ by placing $x$ at the root and attaching depth-$d$ shattered subtrees below the two outgoing edges - contradicting $\mathrm{Ldim}(V)=d$.

Consequently, for every pair $h,c \in V$ we have the analogue of the symmetry inequality,
\begin{equation}\label{eq:symmetry-ldim}
\mathrm{Payoff}(h,c) + \mathrm{Payoff}(c,h) \;\le\; 1,
\end{equation}
since the two events ``$\mathrm{Ldim}(V_{x \to c(x)})=d$'' and ``$\mathrm{Ldim}(V_{x \to h(x)})=d$'' are disjoint for any fixed counterexample~$x$.
As before, linearity extends~\eqref{eq:symmetry-ldim} to mixed strategies, and von Neumann's minimax theorem implies that the learner has a strategy ensuring $\mathrm{Payoff}(h_t,c)\le \tfrac12$ against every $c \in V$.

Interpreting this guarantee, with probability at least $1/2$ the returned counterexample $x$ satisfies
$\mathrm{Ldim}(V_{x \to c(x)}) \le d-1$, and hence after updating the version space to
$V \leftarrow V_{x \to c(x)}$, the Littlestone dimension drops by at least one with constant probability.
It follows that the expected number of rounds until $\mathrm{Ldim}(V)$ reaches~$0$ is $O(\mathrm{Ldim}(\mathcal{H}))$.

\paragraph{Deterministic learners.}
While our minimax argument applies to arbitrary symmetric adversaries, it yields a \emph{randomized} learning rule.
In the special case of \emph{random counterexamples} (i.e., counterexamples drawn from a fixed distribution $\mu$ conditioned on the disagreement set), prior works~\citep{angluin_dohrn_2017,CFR24} provide deterministic learning rules via a more involved graph-theoretic analysis.
Here we outline a short and simple alternative argument which yields a deterministic choice and avoids the technical machinery of these works.

Fix the current version space $V\subseteq \mathcal{H}$ and let $d=\mathrm{Ldim}(V)$. Define, for each $h\in V$, the score
\[
\mathrm{score}(h)
\;:=\;
\Pr_{x\sim \mu}\bigl[\mathrm{Ldim}(V_{x\to h(x)})=d\bigr],
\]
and let $h^\star\in\arg\max_{h\in V}\mathrm{score}(h)$.
We claim that submitting $h^\star$ guarantees constant progress against every target $c\in V$. Fix $c\in V$ with $c\neq h^\star$ and let $E$ denote the event $h^\star(x)=c(x)$.
By the law of total probability,
\begin{align*}
\mathrm{score}(h^\star)
&=
\Pr[E]\cdot \Pr\!\bigl[\mathrm{Ldim}(V_{x\to h^\star(x)})=d \mid E\bigr]
\;+\;
\Pr[\neg E]\cdot \Pr\!\bigl[\mathrm{Ldim}(V_{x\to h^\star(x)})=d \mid \neg E\bigr],\\
\mathrm{score}(c)
&=
\Pr[E]\cdot \Pr\!\bigl[\mathrm{Ldim}(V_{x\to c(x)})=d \mid E\bigr]
\;+\;
\Pr[\neg E]\cdot \Pr\!\bigl[\mathrm{Ldim}(V_{x\to c(x)})=d \mid \neg E\bigr].
\end{align*}
On $E$ we have $h^\star(x)=c(x)$ and hence the first terms in the two displays are identical.
Since $\mathrm{score}(h^\star)\ge \mathrm{score}(c)$ by definition of $h^\star$, cancelling the common first term yields
\begin{equation}\label{eq:score-cancel}
\Pr\!\bigl[\mathrm{Ldim}(V_{x\to h^\star(x)})=d \mid \neg E\bigr]
\;\ge\;
\Pr\!\bigl[\mathrm{Ldim}(V_{x\to c(x)})=d \mid \neg E\bigr].
\end{equation}
Now condition on $\neg E$, i.e., on drawing $x$ from the disagreement set $\{x: h^\star(x)\neq c(x)\}$.
By the key Littlestone observation used above, for any counterexample $x$ with
$h^\star(x)\neq c(x)$, it cannot be the case that both continuations
$V_{x\to h^\star(x)}$ and $V_{x\to c(x)}$ retain full Littlestone dimension~$d$.
Consequently, the conditional probabilities that each of these two continuations
has Littlestone dimension~$d$ sum to at most~$1$.
Combined with~\eqref{eq:score-cancel}, this implies that
\[\Pr\!\bigl[\mathrm{Ldim}(V_{x\to c(x)})=d \mid h^\star(x)\neq c(x)\bigr] \le \tfrac{1}{2}.\]
In other words, when the adversary returns a random counterexample conditioned on
disagreement, the Littlestone dimension drops with probability at least~$1/2$.
This yields a deterministic constant-progress step, and leads to an $O(d)$ bound in expectation in the random-counterexample case.
\subsubsection{Bandit Feedback}
We next outline the proof ideas for the bandit-feedback setting.
Both the algorithm and its analysis are significantly more involved than in the full-information case, and build crucially on the minimax perspective developed above.
Accordingly, we provide only a high-level overview here and refer the reader to the full proof for technical details.



\paragraph{High-level description.}
The main challenge in the bandit-feedback setting is the limited information revealed by each counterexample.
When the learner submits an hypothesis $h_t$ and receives an unlabeled counterexample $x_t$, it only learns that $h_t(x_t)$ is incorrect, without learning the true label~$c(x_t)$. To address this, we maintain a separate hypothesis for \emph{each possible labeling} of the observed counterexamples.
Concretely, suppose the learner has so far received counterexamples $x_1,\ldots,x_t$.
For every label sequence $\vec{y}=(y_1,\ldots,y_t)\in\mathcal{Y}^t$, we define an associated expert that assumes $c(x_i)=y_i$ for all $i\le t$.
This expert corresponds to the version space
\[
V_{\vec{y}} \;=\; \{h\in\mathcal{H} : h(x_i)=y_i \text{ for all } i\le t\},
\]
and predicts using the Standard Optimal Algorithm applied to $V_{\vec{y}}$.\footnote{
Recall that the Standard Optimal Algorithm predicts a label $y$ maximizing $\mathrm{Ldim}(V_{x\to y})$ over all $y'\in\mathcal{Y}$.
}

The learner maintains weights over this evolving collection of experts and updates them upon receiving each new counterexample.
The update rule is guided by three principles:
(i) experts that predict the same label as the learner’s submitted hypothesis on $x_t$ are penalized, since this label is known to be incorrect;
(ii) experts that predict a different label are rewarded, with \emph{smaller-weight experts receiving a larger relative reward}. This reflects the goal of rapidly boosting the weight of the version space consistent with the target concept, even when it initially has negligible mass; and
(iii) experts corresponding to version spaces with larger Littlestone dimension are prioritized, as these retain greater learning potential.

At each round, the learner submits an hypothesis corresponding to the minimax strategy of a game analogous to that defined in the full-information setting, where payoffs are determined by a weighted aggregation of the experts’ predictions.


\paragraph{Analysis overview.}
{The analysis follows a win-win perspective: at each round, either the Littlestone dimension of the expert corresponding to the true label sequence drops, or the weight assigned to that expert increases.
The Littlestone dimension can decrease at most $d=\mathrm{Ldim}(\mathcal{H})$ times.
On the other hand, the reweighting scheme ensures that the weight of the true expert cannot increase too often without triggering such a dimension drop, and in fact this can happen only $\tilde O(k)$ times.
Combining these two effects yields the stated upper bound of $O(d\,k\log k)$ equivalence queries in expectation.
This dichotomy is formally embedded in the analysis via a careful definition of the payoff function in the underlying game on which the minimax strategy is applied.}

\subsection{Lower Bounds}
\paragraph{Full Information.}
For simplicity, consider the binary-labeled case $\mathcal{Y}=\{0,1\}$.
Let $T$ be a shattered mistake tree for $\mathcal{H}$ of maximum depth
$d=\mathrm{Ldim}(\mathcal{H})$. We define a deterministic symmetric adversary $\mathcal{A}^\star$ as follows.
Given a learner hypothesis $h$ and a target $f$, let $\pi_h$ and $\pi_f$ denote the (unique) root-to-leaf paths in $T$ induced by $h$ and $f$, respectively. If $\pi_h=\pi_f$, then $h=f$ on all instances appearing in the tree and the interaction terminates. Otherwise, $\mathcal{A}^\star$ returns the instance labeling the least common ancestor of the two paths, equivalently the node at which the paths first diverge.

To prove the lower bound, fix for each leaf $\ell$ of $T$ an hypothesis
$f_\ell\in\mathcal{H}$ consistent with the root-to-$\ell$ path, and draw the target $f$ uniformly at random from the set $\{f_\ell : \ell \text{ is a leaf of } T\}$. For any learner, its first query $h_1$ induces some path $\pi_{h_1}$. The random target path $\pi_f$ then agrees with $\pi_{h_1}$ for a geometrically distributed number of levels: with probability $1/2$ the paths diverge at the root, with probability $1/4$ they diverge one level below, and so on. Consequently, the expected depth of the returned counterexample - and hence the expected progress down the depth-$d$ tree - is $O(1)$ per round. A standard progress (or optional stopping) argument then implies that the expected number of rounds before reaching depth $d$ is $\Omega(d)$. This establishes the full-information lower bound.

\paragraph{Bandit Feedback.}
We conclude with a lower bound for learning from equivalence queries with bandit feedback.
The construction is based on an hypothesis class that has previously been used to derive lower bounds on the mistake complexity of bandit online learning, which is equivalent to learning from equivalence queries under fully adversarial counterexample generation
\citep{Long20,Geneson21}.
We explain how the same construction yields a matching lower bound in our setting of symmetric (random) counterexample generation.

Fix a prime number $p$, and let the label space be $\mathcal{Y}=\mathbb{F}_p$.
Let the instance space be $\mathcal{X}=\mathbb{F}_p^d$, and let $\mathcal{H}$ be the class of all nonzero linear functionals from $\mathcal{X}$ to $\mathcal{Y}$.
It is known that $\mathrm{Ldim}(\mathcal{H})=d$.
Intuitively, learning in this class with bandit feedback is difficult because each counterexample reveals only that a single label is incorrect, while eliminating only a small fraction of the remaining hypotheses.
To see this, consider the dual viewpoint in which the target hypothesis corresponds to an unknown vector $f \in \mathbb{F}_p^d$.
At each round, the learner proposes an hypothesis $h$, and upon receiving a counterexample $x$, learns only that $h(x)\neq f(x)$.
Equivalently, this removes from consideration a random affine hyperplane that is consistent with $h(x)$ but inconsistent with the true value $f(x)$.
When counterexamples are generated at random, such a step reduces the logarithm of the number of remaining candidate hypotheses by only about $O(1/p)$ in expectation.
Initially, this logarithm is on the order of $d \log p$.
Since $\mathrm{Ldim}(\mathcal{H})=d$ and $|\mathcal{Y}|=p$, this yields a lower bound of
$\Omega\!\bigl(\mathrm{Ldim}(\mathcal{H}) \cdot |\mathcal{Y}| \log |\mathcal{Y}|\bigr)$
equivalence queries with bandit feedback, matching our upper bound up to constant factors.

{The lower bound argument is inherently information-theoretic.
Rather than tracking the evolution of the version space combinatorially, the proof quantifies progress in terms of the mutual information between the observed bandit counterexamples and the unknown target concept.
Each counterexample reveals only that a single label is incorrect, and therefore contributes only a small amount of mutual information about the target hypothesis.
This information-theoretic viewpoint provides a convenient and principled framework for formalizing the intuition that random counterexamples eliminate only a small fraction of the remaining hypotheses at each round, and is what ultimately yields the stated lower bound.}

\section{Additional Related Work}\label{sec:addRelated}
{The classical framework of learning with equivalence queries was introduced by Angluin, who studied exact learning via equivalence and membership queries and established foundational results for this model~\cite{angluin_eq,Angluin87regularsets}. In addition to considering less adversarial counterexample generation, our model also departs from the classical formulation by allowing the learning process to terminate before exact identification of the target concept. A broader and more recent perspective on query learning, including connections to other domains such as formal verification and model inference, is provided by the survey of~\cite{Vaandrager17survey}. A conceptually related but distinct line of work arises in machine teaching, where the feedback provided to the learner is typically non-adversarial and may be chosen cooperatively or strategically by a teacher. We refer to the survey of~\citet*{zhu2018teachingoverview} for an overview of this literature. Particularly relevant to our setting is the work of~\citet*{KumarCS21}, which studies learning with equivalence queries under best-case counterexamples selected by a teacher, in contrast to the worst-case counterexamples considered in the classical model.}

\section{Additional Directions for Future Research}\label{sec:addfuture}

This appendix outlines several additional directions for future research. A first direction concerns relaxing the model to obtain meaningful guarantees for natural hypothesis classes with infinite Littlestone dimension. One avenue arises from restricting attention to \emph{order-induced adversaries}. Recall that such adversaries are specified by a (possibly randomized) sequence of instances and return the first element in the sequence on which the current hypothesis and the target disagree. Intuitively, when the returned counterexample appears far along the underlying sequence - i.e., has large index or complexity - this indicates that the current hypothesis performs well with respect to the ordering.
This intuition is particularly transparent in the case of random counterexamples: if counterexamples are generated by sampling i.i.d.\ from a distribution over instances, then observing that the first counterexample occurs only after $n$ samples suggests that the error of the current hypothesis is on the order of $1/n$. Motivated by this perspective, one can consider the following refined notion.
Given an hypothesis class~$\mathcal{H}$ and an order-induced adversary~$\mathcal{A}$, let $M(n)$ denote the number of equivalence queries required until, with high probability, the adversary returns a counterexample whose index in the underlying sequence is at least~$n$.
When $\mathcal{H}$ has Littlestone dimension~$d$, our results imply that $M(n)=O(d)$ for all~$n$.
More generally, this formulation suggests a way to study learning dynamics for classes with infinite Littlestone dimension, by quantifying how quickly a learner can push counterexamples deeper into the ordering imposed by~$\mathcal{A}$.

Another natural direction is to relax the assumption that the adversary is fully known to the learner.
In many settings, it is more realistic to assume that the learner has only sampling access to the distribution over unlabeled example sequences according to which the adversary operates.
For instance, in the case of random counterexamples, this corresponds to having access to unlabeled samples from the underlying instance distribution.
This leads to a natural tradeoff between equivalence-query complexity and unlabeled sample complexity.
Recent work of~\citep{GU21} provides an encouraging result in this direction, showing that when counterexamples are drawn from an unknown distribution, one can learn up to error~$\varepsilon$ using $\mathrm{polylog}(1/\varepsilon)$ \emph{improper} equivalence queries.

Finally, an additional open problem concerns the role of randomization in learning against symmetric adversaries.
In the special case of random counterexamples, deterministic learning rules are known
\citep{angluin_dohrn_2017,Bhatia22,CFR24}; see also the proof overview below for an alternative short and simple derivation of a deterministic learner in this setting.
In contrast, our general results for symmetric adversaries rely on randomized learning rules obtained via a minimax argument.
It is intetesting to explore whether randomization is inherent in this broader setting.
In particular, is there always a deterministic learning rule for learning against arbitrary symmetric adversaries?
What about the restricted but natural class of order-induced adversaries?

\acks{

RL received funding from the European Research Council (ERC) under the European Union’s Horizon 2020 research and innovation program (grant agreement FoG-101116258).
 Views and opinions expressed are however those of the author(s) only and do not necessarily reflect those of the European Union or the European Research Council. Neither the European Union nor the granting authority can be held responsible for them.

YM received funding from the European Research Council (ERC) under the European Union’s Horizon 2020 research and innovation program (grant agreement No. 882396), by the Israel Science Foundation,  the Yandex Initiative for Machine Learning at Tel Aviv University and a grant from the Tel Aviv University Center for AI and Data Science (TAD).}

SM acknowledges support by the Technion Center for Machine Learning and Intelligent Systems (MLIS), and by the European Union (ERC, GENERALIZATION, 101039692). Views and opinions expressed are however those of the author(s) only and do not necessarily reflect those of the European Union or the European Research Council Executive Agency. Neither the European Union nor the granting authority can be held responsible for them.

\bibliography{bib}

\appendix

\section{Proof of Theorem~\ref{thm:full}}

\begin{theoremrest}[Full-information (restatement of Theorem~\ref{thm:full})]
Let $\mathcal{H}$ be a finite hypothesis class with $\mathrm{Ldim}(\mathcal{H}) = d$, and let $\mathcal{A}$ be a symmetric adversary.
There exists a learning rule such that, for every target concept $c \in \mathcal{H}$, the interaction between the learner and $\mathcal{A}$ terminates after at most $O(d)$ equivalence queries in expectation.

Conversely, there exists a symmetric adversary $\mathcal{A}^\star$ such that, against $\mathcal{A}^\star$, any learning rule requires at least $\Omega(d)$ equivalence queries in expectation.
\end{theoremrest}
We begin by proving the upper bound, and then establish the matching lower bound.
For intuition and a high-level description of the arguments, we recommend that the reader first consult the proof overview in Section~\ref{sec:proof-overview}.

\subsection{Upper Bound}

We prove the upper bound stated in Theorem~\ref{thm:full}.
Throughout this section, let $\mathcal{H}$ be a finite hypothesis class with Littlestone dimension
$\mathrm{Ldim}(\mathcal{H}) = d$, and let $\mathcal{A}$ be a symmetric adversary.
We consider an arbitrary but fixed target concept $c \in \mathcal{H}$.

The learner interacts with $\mathcal{A}$ through equivalence queries.
At each round $t$, the learner proposes an hypothesis $h_t \in \mathcal{H}$.
If $h_t = c$, the adversary must accept and the interaction terminates.
Otherwise, the adversary may return a counterexample $(x_t,c(x_t))$ such that $h_t(x_t) \neq c(x_t)$.

Let $V_t \subseteq \mathcal{H}$ denote the \emph{version space} after round $t$, consisting of all hypotheses
consistent with the counterexamples observed so far.
Initially, $V_1 = \mathcal{H}$.
Upon receiving a counterexample $x_t$, the version space is updated to
\[
V_{t+1} \;:=\; \{ h \in V_t : h(x_t) = c(x_t) \}.
\]
Our goal is to design a learning rule that selects $h_t \in V_t$ so that the expected number of rounds until
$V_t = \{c\}$ is $O(d)$.

\paragraph{Learning rule.}
At each round, given the current version space $V_t$, the learner selects an hypothesis
$h_t \in V_t$ according to a mixed strategy obtained from a zero-sum game defined on $V_t$.
The payoff of the game is chosen so that the learner is penalized whenever the adversary can
return a counterexample that leaves the version space with unchanged Littlestone dimension.

Fix a nonempty version space $V \subseteq \mathcal{H}$, and write $d_V := \mathrm{Ldim}(V)$.
We define a zero-sum game between a learner $L$ and an adversary $A$.
Both players have the same set of pure strategies, namely the hypotheses in $V$.

For $h,c \in V$, define the payoff
\begin{equation}\label{eq:upper-game}
\mathrm{Payoff}(h,c)
\;:=\;
\Pr_{x \sim \mathrm{CE}(c,h \mid \text{history})}
\Bigl[
\mathrm{Ldim}\!\bigl(V_{x \to c(x)}\bigr) = d_V
\Bigr],
\end{equation}
with the convention that $\mathrm{Payoff}(h,h)=0$.
Here $V_{x \to y} := \{h' \in V : h'(x)=y\}$, and the probability is taken over the counterexample
distribution induced by the adversary.

\begin{center}
\fbox{
\begin{minipage}{0.95\linewidth}
\textbf{Algorithm: Full-Information}

\begin{enumerate}
    \item Initialize $V \leftarrow \mathcal{H}$.
    \item While $|V| > 1$:
    \begin{enumerate}
        \item Let $d_V = \mathrm{Ldim}(V)$.
        \item Define the zero-sum game on $V$ with payoff as in~\eqref{eq:upper-game}.
        \item Compute a minimax mixed strategy $p_V$ for the learner.
        \item Sample $h \sim p_V$ and submit $h$ as an equivalence query.
        \item If the adversary accepts, terminate.
        \item Otherwise receive counterexample $(x,y)$ and update
        \[
        V \leftarrow V_{x \to y}.
        \]
    \end{enumerate}
\end{enumerate}
\end{minipage}
}
\end{center}

\begin{lemma}\label{lem:game-value-half}
Fix a nonempty version space $V \subseteq \mathcal{H}$ and let $d_V=\mathrm{Ldim}(V)$.
Consider the zero-sum game on $V$ with payoff defined in~\eqref{eq:upper-game}.
Then the value of the game is at most $\tfrac12$.
In particular, there exists a (possibly randomized)
learner strategy $p_V$ over $V$ such that for every $c \in V$,
\[
\mathrm{Payoff}(p_V,c) \;\le\; \tfrac12.
\]
\end{lemma}

\begin{proof}
We first establish the following symmetry inequality for pure strategies: for all $h,c \in V$,
\begin{equation}\label{eq:symmetry-pure}
\mathrm{Payoff}(h,c) + \mathrm{Payoff}(c,h) \;\le\; 1.
\end{equation}
Fix $h,c \in V$.
If $h=c$, then $\mathrm{Payoff}(h,h)=0$ by definition, and~\eqref{eq:symmetry-pure} holds trivially.
Assume henceforth that $h \neq c$, and write $d_V=\mathrm{Ldim}(V)$. Define
\[
E(h,c)
\;:=\;
\Bigl\{x \in \mathcal{X} : \mathrm{Ldim}\!\bigl(V_{x \to c(x)}\bigr)=d_V\Bigr\},
\qquad
E(c,h)
\;:=\;
\Bigl\{x \in \mathcal{X} : \mathrm{Ldim}\!\bigl(V_{x \to h(x)}\bigr)=d_V\Bigr\}.
\]
We claim that these events are disjoint on the symmetric difference $\{x : h(x)\neq c(x)\}$.
Indeed, fix any $x$ with $h(x)\neq c(x)$.
It cannot be that both $\mathrm{Ldim}(V_{x \to c(x)})=d_V$ and $\mathrm{Ldim}(V_{x \to h(x)})=d_V$ hold simultaneously;
otherwise $V$ would shatter a mistake tree of depth $d_V+1$ by placing $x$ at the root and attaching depth-$d_V$
shattered subtrees below the two outgoing edges, contradicting $\mathrm{Ldim}(V)=d_V$.
Hence, for every $x$ with $h(x)\neq c(x)$,
\[
\mathbf{1}[x\in E(h,c)] + \mathbf{1}[x\in E(c,h)] \;\le\; 1.
\]
Now take expectation with respect to the counterexample distribution
$x \sim \mathrm{CE}(c,h \mid \text{history})$, which is supported on $\{x : h(x)\neq c(x)\}$.
By definition of the payoff~\eqref{eq:upper-game},
\[
\mathrm{Payoff}(h,c)=\Pr_{x \sim \mathrm{CE}(c,h \mid \text{history})}[x\in E(h,c)],
\qquad
\mathrm{Payoff}(c,h)=\Pr_{x \sim \mathrm{CE}(c,h \mid \text{history})}[x\in E(c,h)].
\]
Therefore,~\eqref{eq:symmetry-pure} follows. By linearity,~\eqref{eq:symmetry-pure} extends to mixed strategies: for any distributions $p,q$ over~$V$,
\begin{equation}\label{eq:symmetry-mixed}
\mathrm{Payoff}(p,q) + \mathrm{Payoff}(q,p) \;\le\; 1.
\end{equation}
Now fix any adversary mixed strategy $q$.
Taking $p=q$ in~\eqref{eq:symmetry-mixed} yields
\[
2\,\mathrm{Payoff}(q,q) \;\le\; 1,
\]
and hence $\mathrm{Payoff}(q,q)\le \tfrac12$.
In other words, for every adversary strategy $q$ there exists a learner strategy $p$
(such as $p=q$) satisfying $\mathrm{Payoff}(p,q)\le \tfrac12$,
and hence the \emph{maximin} value of the game is at most $\tfrac12$.

Since the game is finite (both players’ strategy sets are the finite set
$V \subseteq \mathcal{H}$), von Neumann’s minimax theorem applies
\citep{vonNeumann1928,cesaBianchiLugosi2006}.
Therefore, the maximin and minimax values coincide and are equal to the
value of the game.
Consequently, the minimax value of the game is at most $\tfrac12$, and there
exists a learner mixed strategy $p_V$ such that
\[
\mathrm{Payoff}(p_V,c)\le \tfrac12 \qquad \text{for all } c\in V .
\]
\end{proof}

\paragraph{Bounding the number of rounds.}
Let $V_t$ be the version space at round $t$, and write $d_t := \mathrm{Ldim}(V_t)$.
By Lemma~\ref{lem:game-value-half}, in round $t$ the learner chooses $h_t$ so that for every $c\in V_t$,
\[
\Pr\!\Bigl[\mathrm{Ldim}\!\bigl((V_t)_{x \to c(x)}\bigr)=d_t\Bigr] \;\le\; \tfrac12,
\]
where $x \sim \mathrm{CE}(c,h_t \mid \text{history})$ is the counterexample returned by the adversary.
Equivalently,
\[
\Pr[d_{t+1} \le d_t-1] \;\ge\; \tfrac12,
\qquad\text{and always}\qquad d_{t+1}\le d_t.
\]

For $k\in\{0,1,\dots,d\}$ let $E(k)$ denote the worst-case expected number of \emph{additional equivalence queries}
until termination, given that the current version space has Littlestone dimension $k$.
When $k=0$, the version space contains a single hypothesis, so the learner submits it and the adversary must accept; hence
$E(0)=1$.
For every $k\ge 1$, conditioning on whether the dimension drops in the next round yields
\[
E(k) \;\le\; 1 \;+\; \tfrac12 E(k) \;+\; \tfrac12 E(k-1).
\]
Rearranging gives $E(k) \le 2 + E(k-1)$, and therefore by induction $E(k)\le 2k+1$.
In particular, starting from $d=\mathrm{Ldim}(\mathcal{H})$ we obtain $\mathbb{E}[T]\le 2d+1 = O(d)$.

This completes the proof of the upper bound in Theorem~\ref{thm:full}.
\hfill$\square$

\paragraph{Why the argument breaks for infinite classes.}
The above argument relies crucially on the finiteness of the version space~$V$.
Finiteness ensures that the induced zero-sum game is finite, and hence that von Neumann’s minimax theorem applies.
When $\mathcal{H}$ is infinite, this need no longer be the case: the maximin and minimax values of the game may differ,
and a learner strategy guaranteeing bounded payoff against all adversary strategies may fail to exist.

This phenomenon already occurs for very simple infinite classes.
Consider the class of singleton hypotheses over the natural numbers,
$\mathcal{H}=\{h_n : n\in\mathbb{N}\}$, where $h_n(x)=\mathbf{1}[x=n]$,
together with the symmetric, order-induced adversary that, upon rejection of an hypothesis,
returns the minimal counterexample with respect to the natural order on~$\mathbb{N}$.
The resulting interaction induces the classical zero-sum game known as
\emph{guess-the-larger-number},
whose payoff matrix is triangular.
It is well known that this game does not satisfy the minimax theorem: its maximin value is strictly smaller than its minimax value
(see, e.g.,~\citep{HannekeLM21}).

From this perspective, the failure of the argument is not tied to any pathological adversary behavior,
but rather to the absence of a minimax equilibrium in the underlying infinite game.
We note in passing that the random-adversary example discussed in the main results section
also corresponds to a variant of this game, and similarly exhibits a gap between the maximin and minimax values.

\subsection{Lower Bound}

We now prove the lower bound stated in Theorem~\ref{thm:full}.
As in the upper bound, we begin by explicitly defining the adversary.
Let $\mathcal{H}$ be an hypothesis class with $\mathrm{Ldim}(\mathcal{H}) = d$.
Fix a shattered Littlestone tree $T$ of depth $d$ for $\mathcal{H}$.

\paragraph{Associating hypotheses with nodes.}
For each hypothesis $h \in \mathcal{H}$, we associate a node $\nu(h)$ in the tree $T$ as follows.
Starting from the root, we traverse the tree according to the predictions of $h$ for as long as possible.
If at some internal node the prediction of $h$ does not match any of the labels on the outgoing edges,
the traversal stops.
We define $\nu(h)$ to be the deepest node reached in this process (i.e., the node farthest from the root). Note that since $T$ is binary while the label space of $\mathcal{H}$ may be larger, not every hypothesis corresponds to a full root-to-leaf path in $T$;
some hypotheses may halt at internal nodes strictly above the leaves.

\paragraph{Definition of the adversary.}
The lower bound is witnessed by the following tree-based symmetric adversary.

\begin{center}
\fbox{
\begin{minipage}{0.95\linewidth}
\textbf{Adversary $\mathcal{A}^\star$: Tree-Based Symmetric Adversary}

\begin{enumerate}
    \item Fix a shattered Littlestone tree $T$ of depth $d=\mathrm{Ldim}(\mathcal{H})$.
    \item For each hypothesis $h\in\mathcal{H}$, associate a node $\nu(h)$ obtained by traversing $T$
    from the root according to the predictions of $h$, stopping when no outgoing edge is consistent with $h$.
    Let $\nu(h)$ be the deepest node reached.
    \item Upon receiving an equivalence query $h$ with target concept $c$:
    \begin{enumerate}
        \item If \(\nu(h)=\nu(c)\), accept $h$ and terminate.
        \item Otherwise, let $x$ be the instance labeling $\nu:= \mathrm{LCA}(\nu(h),\nu(c))$.
        \item Return $(x,c(x))$ as a counterexample.
    \end{enumerate}
\end{enumerate}
\end{minipage}
}
\end{center}


Since $T$ is shattered, for every root-to-leaf path $\pi$ in $T$ there exists an hypothesis in $\mathcal{H}$
that is consistent with the labels along $\pi$.
Fix, for each leaf $\ell$ (equivalently, each root-to-leaf path), one such hypothesis and denote it by $c_\ell$.
Let
\[
\mathcal{C}_T \;:=\; \{c_\ell : \ell \text{ is a leaf of } T\},
\]
so that $|\mathcal{C}_T| = 2^d$. We consider the following distribution over targets: the target concept $c$ is drawn uniformly at random from $\mathcal{C}_T$.
We will show that for every learning rule, the expected number of equivalence queries against $\mathcal{A}^\star$
(with expectation over the random choice of $c$) is $\Omega(d)$.

\paragraph{Lower bound on the expected number of queries.}
Fix a point during the interaction with the adversary $\mathcal{A}^\star$.
Let $c$ be the target concept and consider the unique root-to-leaf path in the shattered
Littlestone tree $T$ corresponding to $c$. Let $v$ denote the node on the root-to-leaf path in $T$ corresponding to the target concept $c$
that is consistent with all counterexamples returned so far.
Formally, $v$ is defined inductively as follows.
Initially, before any counterexample has been returned, $v$ is the root of $T$.
At any later point at which the adversary has returned at least one counterexample corresponding to a node in $T$,
let $u$ be the deepest node of $T$ whose instance has been returned as a counterexample so far. We then define $v$ to be the child of $u$ that lies on the root-to-leaf path corresponding to $c$ (i.e., the child whose outgoing edge is labeled by $c(u)$).
Let $K_t$ denote the depth of the subtree of $T$ rooted at $v$.
By construction, $K_0=d$, and the interaction terminates exactly when $K_t=0$.
Define
\[
\Delta_t := K_t - K_{t+1}
\]
to be the decrease in depth at round $t$.
By convention, we set $K_t=0$ for all $t\ge T$, where $T$ is the termination time.

Fix a round $t$ with $K_t=k\ge 1$ and condition on the entire history up to round $t$,
including the learner’s (possibly randomized) equivalence query $h_t$.
Let
\[
\nu := \mathrm{LCA}(\nu(h_t),\nu(c))
\]
be the node chosen by the adversary $\mathcal{A}^\star$.

For $\Delta_t$ to be at least $i$, it is necessary that $\nu$ lies inside the current depth-$k$ subtree
and is at least $i-1$ edges below its root.
In particular, this requires that $h_t$ agrees with the target $c$ on the first $i-1$ nodes
along the branch of $T$ corresponding to $c$.
Since $c$ is uniform over the $2^k$ leaves of the current subtree,
this agreement event has probability at most $2^{-(i-1)}$.
Therefore, for all $i\in\{1,\dots,k\}$,
\[
\Pr(\Delta_t \ge i \mid K_t=k) \;\le\; 2^{-(i-1)}.
\]
It follows that
\[
\mathbb{E}[\Delta_t \mid K_t=k]
\;=\;
\sum_{i=1}^{k} \Pr(\Delta_t \ge i \mid K_t=k)
\;\le\;
\sum_{i=1}^{k} 2^{-(i-1)}
\;\le\; 2.
\]
Let $T$ denote the (random) termination time.
Since $\Delta_t=0$ for all $t\ge T$, we have
\[
\mathbb{E}[\Delta_t \mid \mathcal{F}_t] \;\le\; 2\,\mathbf{1}\{t<T\},
\]
where $\mathcal{F}_t$ denotes the history up to round $t$.
Taking expectations and summing over all $t\ge 0$ yields
\[
d
\;=\;
\mathbb{E}[K_0-K_T]
\;=\;
\sum_{t\ge 0} \mathbb{E}[\Delta_t]
\;\le\;
2 \sum_{t\ge 0} \Pr(t<T)
\;=\;
2\,\mathbb{E}[T].
\]
Hence $\mathbb{E}[T]\ge d/2$.

We conclude that against the adversary $\mathcal{A}^\star$, any learning rule requires
$\Omega(d)$ equivalence queries in expectation.
This completes the proof of the lower bound.
\hfill$\square$

\section{Proof of Theorem~\ref{thm:bandit}}

This is a complete version of a proof sketched in the main text.

\subsection{Upper Bound}
Similar to the full information setting, our strategy is again built on a minmax strategy in a payoff game. However, unlike in the full information setting, we cannot reduce the version space based on our partial feedback. As such, the payoff matrix at round $t$ is defined in a different manner.

\paragraph{The payoff matrix at round $t$:}
At step $t$, consider the sequence $\{x_1,\ldots, x_t\}$ of examples we've seen up to round $t$ and assume we also have a distribution $\mu_t$ on seqeunces in $[k]^t$ which we maintain at each round.

For every hypothesis we denote: 
$$
L_t(h):=Ldim_t\left(\left\{h'\in \mathcal{H},\quad h'(x_i)=h(x_i), \forall i=1,\ldots, t\right\}\right).
$$
$$
L^x_t(h):=Ldim_t\left(\left\{h'\in \mathcal{H},\quad h'(x_i)=h(x_i), \forall i=1,\ldots, t, \wedge  h'(x)=h(x)\right\}\right).
$$

We further add a notation, $r_{s,t}(x)$, for the decision rule of the \emph{Standard Optimal Algorithm} on input $x$ with labels $s\in[k]^t$ at time $t$, and we let 
\[r_{s,t}(x)=
\begin{cases} y & \exists h\in \mathcal{H},\quad h(x_i)=s_i,~L^x_t(h)=L_t(h), \textrm{~and~} h(x)=y \\
0 & \textrm{else}
\end{cases}
.\]

Next, we define the update rule over the distribution $\mu_t$ that we maintain, given $x=x_{t+1}$ and an input labeling $i$:
\begin{center}
\fbox{
\begin{minipage}{0.95\linewidth}
\textbf{Update rule: $\mu^{(i)}_{t+1}$ from $\mu_t$}
\begin{enumerate}
    \item Input: $x=x_{t+1}$ and $i\in [k]$.
    \item For each $j\in [k]$ assign $\kappa_t(x,j):=\mu_t\{ s~:~r_{s,t}(x)=j\}$.
    \item For a string of values $(s,j)$ let 
    $$
    \nu(s,j)=\left\{
\begin{array}{ll}
 0 & \text{if $j=i$}\\
   \mu_t(s)\cdot \left(1-\frac{k-1}{k^3}+\frac{\kappa_t(x,i)}{\kappa_t(x,j)\cdot k}\right)  & \text{if $r_{s,t}(x)=j$ (and $j\ne i$)}  \\
     \mu_t(s)\cdot \frac{1}{k^3}  &  \text{otherwise}
\end{array}
\right.$$
\item 
 Normalize $\nu$ to obtain probability distribution $\mu_{t+1}$: 
 $$
 \mu^{i}_{t+1}(s,j):= \nu(s,j)\cdot \left(\sum_{s',j'} \nu(s',j')\right)^{-1}
 $$
\end{enumerate}
\end{minipage}
}
\end{center}
Next, we let:
\[ \Phi_{t}(h;\mu_{t}) = (8\ln k)L_t(h)-\ln \mu_t(h),\quad \Phi^{(x,i)}_{t}(h;\mu_t) = (8\ln k)L^x_t(h)-\ln \mu^i_{t+1}(h) .\]
Then, the payoff matrix at round $t$ is defined as follows $\textrm{Payoff}(h,c)=-1$ if $h=c$ and otherwise:
\begin{align*}\textrm{Payoff}(h,c)& :=\E_{\ce}\left(\textrm{Payoff}^x(h,c)\right)\\&:= \E_{\ce}\left(\Phi^{(x,h(x))}_{t}(c;\mu_t)-\Phi_t(c;\mu_t)\right)\labelthis{eq:payoff_partial}.\end{align*}

\paragraph{Optimal Strategy in the payoff game}
We next show that a learner that wants to minimize the payoff, has an optimal strategy with bounded, negative, payoff. As before, we rely on a maxmin argument and the strategy of an imitating player. Therefore, the conclusion will follow from the following claim:
\begin{claim}\label{cl:1:wip}

Let $h,c\in \mathcal{H}$ be  two hypotheses. Let $x\notin\{x_1,\ldots,x_t\}$ be such that $h(x)\neq c(x)$. Then 
\[
\mathrm{Payoff}^x(h,c)+\mathrm{Payoff}^x(c,h)\le -\frac{1}{6k}
\] 
\end{claim}
Before we prove the game let us observe how it entails an optimal strategy for the learner. Indeed, consider the payoff matrix defined at round $t$. We can show that there exists $p_t$ such that:

\begin{equation}\label{eq:minmax_partialfeed}
\min_{p_t} \max_{q} \E_{\ce} \left[\textrm{Payoff}(p_t,q)\right] \le -\frac{1}{12k}
\end{equation}
To see that, we again use the minimax theorem. Thus, we only need to show that for any mixed strategy $q$ for the adversary, the learner can guarantee a payoff of at most $-\frac{1}{12k}$. 

    Consider the imitation strategy $p=q$. The learner's payoff in this case is, by symmetry and Claim~\ref{cl:1:wip}:
    \begin{align*}
    \mathrm{Payoff}(q,q) =& \E\left(\mathrm{Payoff}^x(h,c)\right)
    \\  <& -\sum_{h}q^2(h)+ \frac{1}{2}\cdot \sum_{c\neq h} q(h)\cdot q(c) \cdot \left(\E_{\ce}\left[\mathrm{Payoff}^x(h,c)+\mathrm{Payoff}^x(c,h)\right]\right)
    \\&\le -\sum_{h}q^2(h) -\frac{1}{12k}\sum_{c\ne h}q(h)q(c)\\
&\le -\frac{1}{12k}
    \end{align*}

\paragraph{Proof of \cref{cl:1:wip}}
    We first observe that $\nu(s,j)$ adds up to at most $1$. For a fixed $s$, denote $j^s=r_{s,t}(x)$. Then for every $s$ such that $j^s\ne i$:

    \begin{align*}
        \sum_{j}\nu(s,j) &\le  \mu_t(s)\left(1-\frac{k-1}{k^3} + \frac{\kappa_t(x,i)}{\kappa_t(x,j^s)\cdot k}\right)+ \sum_{j\ne j^s} \frac{1}{k^3}\cdot \mu_t(s)\\
        & = \mu_t(s)\left(1 + \frac{\kappa_t(x,i)}{\kappa_t(x,j^s)\cdot k}\right),
    \end{align*}
    and if $j^s = i$, then 
    \[\sum_{j}\nu(s,j) = \frac{k-1}{k^3} \mu_t(s).\]
    therefore,
    \begin{align*}
        \sum_s\sum_{j}\nu(s,j) &= \sum_{j'}\sum_{\{s: j^s=j'\}}\left(\sum_{j} \nu(s,j)\right)\\ 
        &= \sum_{\{s: j^s=i\}}\left(\sum_{j} \nu(s,j)\right)+ \sum_{j'\ne i}\sum_{\{s: j^s=i\}}\left(\sum_{j} \nu(s,j)\right)\\
        &\le \sum_{\{s: j^s=i\}}\left(\frac{k-1}{k^3}\mu_t(s)\right)+ \sum_{j'\ne i}\sum_{\{s: j^s=j'\}}\left(1 + \frac{\kappa_t(x,i)}{\kappa_t(x,j^s)\cdot k}\mu_t(s)\right)\\
        &= \frac{k-1}{k^3}\mu_t(\{s: j^s=i)\}) + \sum_{j'\ne i} 
         \left(1+ \frac{\kappa_t(x,i)}{\kappa_t(x,j^s)\cdot k}\right)\mu(\{s: j^s=j'\})\\  
         &=\frac{k-1}{k^3}\kappa_t(x,i) +\sum_{j\ne i} \kappa_t(x,j) +\sum_{j\ne i}\frac{\kappa_t(x,i)}{k\cdot\kappa_t(x,j)}\kappa_t(x,j)& \kappa_t(x,j)=\mu(\{s: j^s=j\})\\
        &= \frac{k-1}{k^3}\kappa_t(x,i) +(1-\kappa_t(x,i))+\frac{k-1}{k}\kappa_t(x,i)\\
        &\le 1 +\left(\frac{1}{k^2}-\frac{1}{k}\right)\\
        &\le 1
    \end{align*}
 
 We now consider two cases:
 
\paragraph{Case 1. $L^x_{t}(c)<L_t(c)$ 
or  $L^x_t(h)<L_t(h)$.}

Observe that since $c(x)\neq h(x)$, we always have 
$\mu_{t+1}^{(c(x))}(h)\ge  k^{-3}\cdot \mu_t(h)$, and similarly 
$\mu_{t+1}^{(h(x))}(c)\ge  k^{-3}\cdot \mu_t(c)$. Therefore
\begin{align*}
\mathrm{Payoff}^x(h,c)+\mathrm{Payoff}^x(c,h)&\le 2\left(3\ln k\right) + (8\ln k)\left(L^x_{t}(c)-L_t(c)+L^x_{t}(h)-L_{t}(h)\right)\\
&\le 6\ln k - 8\ln k\\
&= -2\ln k
\\& \le -\frac{1}{5k}
\end{align*}
where we used the fact that $L_t(c)+L_t(h)\geq 1+L^x_{t}(h)+L^x_{t}(c) $.
\paragraph{Case 2. $L^x_{t}(c)=L_t(c)$ 
and  $L^x_{t}(h)=L_t(h)$.} In this case we have 
$r_{h,t}(x)=h(x)$ and $r_{c,t}(x)=c(x)$. Thus,
\begin{align*}
\mathrm{Payoff}^x(h,c)&+\mathrm{Payoff}^x(c,h)\\ &= \ln \mu_t(h)-\ln\mu_{t+1}^{c(x)}(h)+
\ln \mu_t(c)-\ln\mu_{t+1}^{h(x)}(c)\\
&\le-\ln\left(1-\frac{k-1}{k^3}+\frac{\kappa_t(x,c(x))}{\kappa_t(x,h(x))\cdot k}\right)-
\ln\left(1-\frac{k-1}{k^3}+\frac{\kappa_t(x,h(x))}{\kappa_t(x,c(x))\cdot k}\right)\\
&\le -\ln(1-\frac{1}{k^2}+\frac{1}{k^3})-\ln (1+\frac{1}{k}-\frac{1}{k^2}+\frac{1}{k^3})&  \\
& \le-\ln(1-\frac{1}{2k^2})-\ln (1+\frac{1}{2k}) \\
& \le -\ln (1+\frac{1}{2k}-\frac{1}{2k^2}-\frac{1}{4k^3})\\
&\le -\ln \left(1+\frac{3}{16k}\right)\\
& \le -\frac{1}{6k}
\end{align*}
Where the second inequality is because \[\max\left\{\frac{\kappa_t(x,c(x))}{\kappa_t(x,h(x))},\frac{\kappa_t(x,h(x))}{\kappa_t(x,c(x))}\right\}\ge 1.\]
\hfill$\square$

\paragraph{The Algorithm}    
We are now ready to give the full algorithm. 

\begin{center}
\fbox{
\begin{minipage}{0.95\linewidth}
\textbf{Algorithm: Bandit-Feedback}

\begin{enumerate}
    \item Initialize $\mu_0$ to be the unit mass on empty string.
    \item While the adversary doesn't accept:
    \begin{enumerate}
        \item Define the zero-sum game on $\mathcal{H}$ with payoff as in \cref{eq:payoff_partial}.
        \item Compute a minimax mixed strategy $p_t$ for the learner as in \cref{eq:minmax_partialfeed}.
        \item Sample $h \sim p_t$ and submit $h$ as an equivalence query.
        \item If the adversary accepts, terminate.
        \item Otherwise receive counterexample $x_{t+1}=x$ and update $\mu_{t+1}:=\mu^{h(x)}_{t+1}$ according to the update rule
    \end{enumerate}
\end{enumerate}
\end{minipage}
}
\end{center}

\begin{claim}
    The expected number of steps of the the Bandit-Feedback algorithm is $\le  100 d k \ln k$.
\end{claim}
\begin{proof} Fix the concept $c$ the algorithm  is trying to learn. Let $T$ be a random variable representing the time step at which the algorithm terminates. 
    Define the following random process $V_t$.
    $$
    \Phi_t := 
    \left\{
\begin{array}{ll}
   \Phi_t(c;\mu_t)  & \text{ if $t<T$, that is, the algorithm is still active}  \\
   -1  &  \text{otherwise (algorithm terminated)}
\end{array}
\right.
    $$
Observe that assuming that at step $t$ the algorithm is active, then at step $t$:

$$
\Phi_{t+1}-\Phi_t \le \mathrm{Payoff}(h,c). 
$$
If we do not terminate after step $t$ then in particular, $h\neq c$, and this is true (with equality) by the definition of $\mathrm{Payoff}(h,c)$, and $\Phi_{t+1}=\Phi_{t+1}(c;\mu_t)=\Phi^{(x,h(x))}_{t}(c;\mu_t)$. 

If we terminate after step $t$, then $\Phi_{t+1}=-1\le \Phi_{t}^{(x,h(x))}(c;\mu_t)$, 
and the inequality holds

Thus, assuming $t<T$, by \cref{eq:minmax_partialfeed} we have 
$$
\mathbf{E}[\Phi_{t+1}-\Phi_t] \le -\frac{1}{12 k},
$$
and 
$$
\mathbf{E}[\Phi_{t+1}-\Phi_t] \le -\frac{1}{12 k}\cdot \Pr[T>t]. 
$$
Thus, for every $N>0$, 
\begin{align*}
    \mathbf{E}[\min(T,N)] = \sum_{t=0}^{N-1} \Pr[T>t] \le&
    10 k \cdot \sum_{t=0}^{N-1}\mathbf{E}[\Phi_{t}-\Phi_{t+1}]\\ =&  
     10 k \cdot\mathbf{E}[\Phi_{0}-\Phi_{N}]  \le 12 k \cdot \left(11d\ln k \right) < 120 d k \ln k. 
\end{align*}
Thus $$  \mathbf{E}[T]\le  120 d k \ln k$$
\end{proof}


$$
$$
\subsection{Lower bound}

\paragraph{Proof overview.} The concept class we consider are linear function over $\mathbb{F}_p^d$. We consider the simple model where counterexamples $x$ given linear functions $h$ and $c$ come from the uniform distribution on their disagreement points. Any two linear functions $h$ and $c$ disagree on all but $p^{d-1}$ points, which makes $x$ nearly uniformly distributed over $\mathbb{F}_p^d$. A uniform $x$ that does not depend on $(h,c)$ would have revealed no information about $c$. An ``almost uniform" $x$ reveals very little information about $c$ (namely, $O(1/p)$ bits). Initially, $c$ has $\Theta(d\log p)$ bits of entropy, leading to a lower bound of $\Theta(d p\log p)$ on the necessary number of queries. 

\smallskip

Given parameters $d$ and $k$, 
fix a prime number $p=\Theta>10$, and let the label space be $\mathcal{Y}=\mathbb{F}_p$.
Let the instance space be $\mathcal{X}=\mathbb{F}_p^d$, and let $\mathcal{H}$ be the class of all nonzero linear functionals from $\mathcal{X}$ to $\mathcal{Y}$.
It is standard that $\mathrm{Ldim}(\mathcal{H})=d$: each new value fixes the value of $f$ on one more dimension.

We consider the simplest adversary, which given a concept $c$ and hypothesis $h\neq c$ outputs a random point in 
$$D(c,h):=\{x~:~c(x)\neq h(x)\}.$$
Note that $D(c,h)$ covers almost all of $\mathcal{X}$: we have $|D(c,h)|=(p-1)\cdot p^{d-1}$.

Suppose there is an algorithm for learning $\mathcal{H}$ in the bandit feedback model using at most $q$ queries in expectation. If we choose the concept $c$ at random from $\mathcal{H}$, we can make the algorithm {\em deterministic}: the next query $h_t$ is a deterministic function of the transcript. 

Let $Q$ be the random variable representing the number of queries the algorithm ended up asking. We represent the transcript $\Pi$ of the learning protocol as a sequence of triples $(H_t,S_t,X_t)$: $H_t$ is the hypothesis the learner asks at step $t$, $S_t\in\{0,1\}$ is the flag representing whether $H_t=c$ (thus $S_Q=1$, and $S_t=0$ for $t<Q$), and $X_t$ is the random example distinguishing $H_t$ from $c$ the adversary returns. For $t>Q$ we define this tuple to be empty (since the algorithm terminates). 

Suppose that $\mathbf{E}[Q]=q$, by Markov's inequality $\Pr[Q\le 2q]\ge 1/2$, Denote by $W$ the indicator random variable such that $W=1$
if $Q\le 2q$, so that $Pr[W=0]\le 1/2$.

We use $\mathbf{H}(\cdot)$ to denote Shannon's entropy of a random variable, $\mathbf{I}(\cdot~;\cdot)$ is the mutual information.  Note that $\Pi_{1..2q}$ determines $W$, and thus $\mathbf{H}(W|\Pi_{1..2q})=0$. Let $C$ be the random variable representing the hidden concept. We have 
\ignore{\ym{The first line mkes no sense. Are you trying to say that \[E_{\Pi_{1..2q}}[\mathbf{I}(\Pi_{1..2q}; C)] = \mathbf{H}(C)-
    E_{W,\Pi_{1..2q}}[\mathbf{H}(C|\Pi_{1..2q},W)] \]
    I assume you are using the fact that $Pr[W=0]\leq 1/2$ and $\mathbf{H}(C|\Pi_{1..2q},W=1)=0 $. Is the end simply that $\mathbf{H}(C|\Pi_{1..2q},W=0) \leq \mathbf{H}(C) =d\log p$? }}
\begin{equation}
    \label{eq:B1}
    \begin{aligned}
     \mathbf{I}(\Pi_{1..2q}; C) &= 
      \mathbf{H}(C)-
       \mathbf{H}(C|\Pi_{1..2q}) & \text{(by definition)} \\
       & \ge  \mathbf{H}(C)-
       \mathbf{H}(C W|\Pi_{1..2q}) & \text{(adding term to entropy)}
       \\ & = \mathbf{H}(C)-
       \mathbf{H}(W|\Pi_{1..2q})-
       \mathbf{H}(C|\Pi_{1..2q}W)
       & \text{(chain rule)}
       \\ ~
       & =   \mathbf{H}(C)-
       \mathbf{H}(C|\Pi_{1..2q}W) 
         & (\mathbf{H}(W|\Pi_{1..2q})=0 )
         \\ &=
       \mathbf{H}(C)-Pr[W=0]\cdot 
       \mathbf{H}(C|\Pi_{1..2q},W=0)\\&\;\;\; -
       Pr[W=1]\cdot 
       \mathbf{H}(C|\Pi_{1..2q},W=1) 
       & \text{(splitting on value of $W$)} \\& =  
       \mathbf{H}(C)-Pr[W=0]\cdot 
       \mathbf{H}(C|\Pi_{1..2q},W=0)
       &\text{(when $W=1$, $\Pi$ determines $C$)}\\ & \ge (d \log p) -  Pr[W=0]\cdot d\log p & \ignore{\text{entropy $\le\log(|\text{domain}|)$}} \text{($ \mathbf{H}(C|\cdot)\leq\mathbf{H}(C)=d\log p$)} 
\\ & \ge (d \log p)/2 & (Pr[W=0]\le 1/2)
         \end{aligned}
\end{equation}
    \ignore{
\begin{multline}
    \label{eq:B1111}
    \mathbf{I}(\Pi_{1..2q}; C) =   \mathbf{I}(\Pi_{1..2q}; C) +\mathbf{I}(W; C|\Pi_{1..2q})
    =   \mathbf{I}(\Pi_{1..2q}W; C) \\=  \mathbf{H}(C)-
    \mathbf{H}(C|\Pi_{1..2q},W) \ge d\log p - Pr[W=0]\cdot  \mathbf{H}(C|\Pi_{1..2q},W=0) \ge 
    ( d\log p)/2. 
\end{multline}}

\ignore{
\textcolor{red}{I don't follow why $\mathbf{H}(C|\Pi_{1..2q}) \leq (d\log p)/2 $. However, I do see why $\mathbf{H}(C|\Pi_{1..2q}) \leq 1 + (d\log p)/2$, which is sufficient. The latter follows by letting $\mathtt{I}$ be the indicator of the event $Q\le 2q$,
and by:}
\begin{align*}
\textcolor{red}{\mathbf{H}(C|\Pi_{1..2q})} &\textcolor{red}{\leq \mathbf{H}(\mathtt{I},C|\Pi_{1..2q})}\\
                          &\textcolor{red}{= \mathbf{H}(\mathtt{I}|\Pi_{1..2q}) + \mathbf{H}(C|\Pi_{1..2q},\mathtt{I})}\\
                          &\textcolor{red}{\leq  1 + \frac{1}{2}\mathbf{H}(C|\Pi_{1..2q},\mathtt{I}=0)}\\
                          &\textcolor{red}{\leq 1 + \frac{1}{2}\mathbf{H}(C).}
\end{align*}
}

In other words, the first $2q$ elements of the transcript reveal at least half the information about the hidden concept $C$. 

On the other hand, by chain rule:
\begin{equation}
    \label{eq:B2}
    \mathbf{I}(\Pi_{1..2q}; C) =\sum_{t=1}^{2q}
     \mathbf{I}(\Pi_{t}; C|\Pi_{<t}). 
\end{equation}
We have (recall that $H_t$ is deterministic given the transcript):
\begin{align*}
     \mathbf{I}(\Pi_{t}; C|\Pi_{<t}) =&
        \mathbf{I}(H_{t}; C|\Pi_{<t})+
        \mathbf{I}(S_{t}; C|\Pi_{<t} , H_t)+
        \mathbf{I}(X_{t}; C|\Pi_{<t} , H_t ,  S_t) 
        \\ \le& \mathbf{H}(H_t|\Pi_{<t}) + 
        \mathbf{H}(S_t|\Pi_{<t} , H_t)  +
        \mathbf{I}(X_{t}; C|\Pi_{<t} , H_t, S_t=0) \\
        \le& 0 +  \mathbf{H}(S_t| S_{<t}) + 
        \mathbf{H}(X_{t}|\Pi_{<t}H_t, S_t=0)-
        \mathbf{H}(X_{t}|\Pi_{<t}H_t, S_t=0,C)\\ 
        \le&  \mathbf{H}(S_t| S_{<t}) + 
        \mathbf{H}(X_{t})-
        \mathbf{H}(X_{t}|\Pi_{<t}H_t, S_t=0,C)\\
    =& \mathbf{H}(S_t| S_{<t}) + \log (p^d) - 
        \log ((p-1)\cdot p^{d-1}) \\ =&
         \mathbf{H}(S_t| S_{<t}) + \log (p/(p-1)). 
\end{align*}
The second to last equality is because, given $C$ and $H_t$, the entropy of the example  $X_t$ is  $\log ((p-1)\cdot p^{d-1})$.
Plugging this into \eqref{eq:B2}, we get
\begin{align*}
      \mathbf{I}(\Pi_{1..2q}; C) \le& \sum_{t=1}^{2q} \left[    \mathbf{H}(S_t; S_{<t}) + \log (p/(p-1))\right] \\=& 
  \mathbf{H}(S_{1..2q}) + 2q \log (p/(p-1)) \le  \log (2q+1) + \frac{2q}{p-1}<\frac{3q}{p}. 
\end{align*}

Plugging this into \eqref{eq:B1}, we get
$$
\frac{3q}{p}> \mathbf{I}(\Pi_{1..2q}; C)\ge 
\frac{d\log p}{2},
$$
implying the lower bound on the expected number of queries:
$$
q> (d\cdot p \log p)/6 =\Theta(d\cdot k \log k). 
$$

\section{Adaptive Adversaries}\label{sec:adaptiveadv}
The following is a variant of learning from equivalence queries, allowing for an adversarial re-selection of the target concept in each round, subject to constraints created in previous rounds.

\begin{mdframed}[
  backgroundcolor=gray!10,
  linecolor=gray!50,
  linewidth=0.6pt,
  roundcorner=6pt,
  innertopmargin=1em,
  innerbottommargin=1em
]
\begin{center}
\textbf{Learning from Equivalence Queries with Concept Reselection}
\end{center}

\medskip
\noindent
Let $\mathcal{H}$ be a hypothesis class over a domain $\mathcal{X}$ with a label space $\mathcal{Y}$.

\medskip
\noindent
Let $\mathcal{H}_1=\mathcal{H}$, and let $c_1 \in \mathcal{H}_1$ be an unknown target concept.

\medskip
\noindent
The learning interaction proceeds in rounds $t = 1,2,\ldots$ as follows:
\begin{itemize}
    \item The learner proposes an hypothesis $h_t \in \mathcal{H}_t$.
    \item The environment either \emph{accepts} or \emph{rejects} $h_t$.
    \item If the environment accepts $h_t$, the interaction terminates.
    \item If the environment rejects $h_t$, it returns a counterexample $(x_t,y_t) = (x_t,c_t(x_t))$ s.t.\
    $h_t(x_t) \neq c_t(x_t)$.
    \item Let $\mathcal{H}_{t+1}=\{h\in\mathcal{H}_t~:~ h(x_t)=y_t\}$ and let $c_{t+1}$ be a target concept (potentially different from $c_t$) chosen from $\mathcal{H}_{t+1}$.
\end{itemize}

\noindent
The environment is required to accept $h_t$ whenever $h_t = c_t$.

\medskip
\noindent\textit{Interpretation.}
The target concepts $c_1,c_2,\ldots$ may be selected adversarially and adaptively (i.e., choice of $c_{t+1}$ may depend on the previous $t$ round of the protocol) subject to the requirement that target concept $c_{t+1}$ agrees with the labels provided in rounds $1,\ldots,t$.
\end{mdframed}

{Observe that, as long as counterexample generation is symmetric, our proofs also hold for learning under an adaptive adversary.}

\end{document}